\documentclass[12pt]{article}
\usepackage{amsmath,amssymb}
\usepackage{spalign}

\newtheorem{thm}{Theorem}[section]
\newtheorem{lem}[thm]{Lemma}

\newcommand{\NN}{\mathcal{N}}

\newcommand{\R}{\mathbb{R}}

\newcommand{\st}{\,\,:\,\,}
\newcommand{\Section}[1]{\section{#1}\setcounter{equation}{0}}

\renewcommand{\vec}[1]{\boldsymbol{#1}}
\newcommand{\vb}{{\vec{b}}}
\newcommand{\va}{{\vec{a}}}
\newcommand{\vx}{\vec{x}}
\newcommand{\vy}{\vec{y}}

\newcommand{\openbox}{\leavevmode
  \hbox to.77778em{%
    \hfil\vrule
  \vbox to.675em{\hrule width.6em\vfil\hrule}%
  \vrule\hfil}}

\newcommand{\proofname}{Proof}
\newenvironment{proof}[1][\proofname]{\par\normalfont
  \trivlist\item[\hskip\labelsep\itshape #1:]\ignorespaces
  }{\hspace*{1cm}\hspace*{\fill}\openbox \medskip\endtrivlist}

\title{An elementary proof of a universal approximation theorem
}

\date{}
\author{Chris Monico \\
  {\small Department of Mathematics and Statistics\vspace{-2mm}}\\
  {\small Texas Tech University\vspace{-2mm}}\\
  {\small {\em e-mail:\/} c.monico@ttu.edu }
  }

\begin{document}\maketitle
\thispagestyle{empty}
\begin{abstract}
  In this short note, we give an elementary proof of a universal
approximation theorem for neural networks with three hidden
layers and increasing, continuous, bounded activation function.
The result is weaker than the best known results, but the proof
is elementary in the sense that no machinery beyond undergraduate analysis
is used.
\end{abstract}

\Section{Introduction}  \label{sec:intro}

  There are several versions of {\em universal approximation theorems}
known, including the very well-known ones from \cite{cyb89, hor89, les93}.
 Each of them states that some collection of neural networks
is dense in some space of continuous functions with respect to the
uniform norm. In this short note, we present what we believe to be a new
and atypically elementary proof of one such theorem. 
If $\sigma$ is a 0-1 squashing function (a.k.a. a {\em sigmoidal function}), we show that the collection
of neural networks with three hidden layers and activation function $\sigma$
(except at the output) is dense in the space $C(K)$ of real-valued continuous
functions on a compact set $K\subset\R^n$.
The result given here is weaker than the best known results,
but the argument relies only on basic results about compact sets and continuous functions
which are generally covered in an undergraduate analysis course;
it is really nothing more than an exercise in ``epsilon chasing,''
though the underlying intuitive motivation is fairly natural.

\section{Notations}

  For the entirety of this note, we assume the following:
  \begin{enumerate}
    \item $\sigma:\R\longrightarrow\R$ is increasing, continuous, and
          $\displaystyle{\lim_{x\to -\infty}\sigma(x)=0}$ and 
          $\displaystyle{\lim_{x\to\infty}\sigma(x)=1}$.
    \item $K$ is a compact subset of $\R^n$.
  \end{enumerate}
  For example, the sigmoid function $\sigma(x)=1/(1+e^{-x})$ satisfies the first
  assumption. A function satisfying that assumption will be called a 0-1 {\em squashing function}.
  There's nothing particularly special about 0 and 1 here. The crucial
  properties used in our arguments are really just that $\sigma$ is
  increasing, continuous, bounded, and non-constant. 
  But the discussion is simplified
  with these additional limiting assumptions.

  The compact set $K$ will represent the domain of a function to be
  approximated. In particular, the unit cube $K=[0,1]^n$ is a special case.

  Of course, it's not at all necessary for a neural network to use the same
  activation function in each hidden layer, but that will suffice for our
  purposes. Fix a positive integer $n$ and let 
  \begin{eqnarray*}
    \NN_1 &=& \{ f\in C(K) \st f(x_1,\ldots,x_n) = a_0+a_1x_1 + \cdots + a_nx_n, \\
          &&       \hspace{12pt}\mbox{ for some }a_0,\ldots,a_n\in\R\},\\
    \NN_1^\sigma &=& \{ F\in C(K) \st F = \sigma\circ f, \hspace{12pt}\mbox{ for some } f\in\NN_1\}.
  \end{eqnarray*}
  That is, $\NN_1$ is the set of all affine functions of $x_1,\ldots, x_n$.
  And $\NN_1^\sigma$ is the set of all possible node output functions in Layer 1.
  For $k\ge 1$, we define
  \begin{eqnarray*}
    \NN_{k+1} &=& \{ g\in C(K) \st g = a_0+a_1F_1 + \cdots + a_mF_m, \\
              &&    \hspace{12pt}\mbox{ for some } F_1,\ldots,F_m\in\NN_k^\sigma, \,\,   a_0,\ldots,a_m\in\R\},\\
    \NN_{k+1}^\sigma &=& \{ G\in C(K) \st G = \sigma\circ g, \hspace{12pt}\mbox{ for some } g\in\NN_{k+1}\}.
  \end{eqnarray*}
  Thus, $\NN_{k+1}^\sigma$ is the set of all possible node output functions
  in Layer $k+1$.
  It follows easily from the definitions above that $\NN_k^\sigma \subset \NN_{k+1}$,
  and if $g_1, g_2\in\NN_{k}$ then $(a_0 + a_1g_1 + a_2g_2)\in\NN_{k}$ for all
  $a_0,a_1,a_2\in\R$. We will use these facts frequently in the sequel.

\section{Three separation lemmas}

  Roughly speaking, the strategy of the proof is to first show that 
  $\sigma$ separates points in $\R^n$ in a strong sense.
  We then show that $\NN_2$
  separates points from closed sets in a similarly strong sense, and that
  $\NN_3$ separates closed sets from closed sets again in a strong sense. 
  Finally, the theorem will be proven using this result
  about $\NN_3$.

\begin{lem} \label{lem:ptpt}
  Let $x_0$ and $x_1$ be distinct real numbers.
  For each $\epsilon>0$ there exist $s,t\in\R$ such that $\sigma(s+tx_0)<\epsilon$
  and $\sigma(s+tx_1)>1-\epsilon$. If, in addition, $x_0<x_1$ and $\epsilon<1/2$, 
  then $\sigma(s+tx)<\epsilon$ on the interval $(-\infty, x_0]$ and 
  $\sigma(s+tx)>1-\epsilon$ on the interval $[x_1,\infty)$.
\end{lem}
\begin{proof}
  Without loss of generality, we may assume $\epsilon<1$.
  By definition of a 0-1 squashing function, there exist $y_0,y_1\in\R$ such
  that $\sigma(y_0)=\epsilon/2$ and $\sigma(y_1)=1-\epsilon/2$. 
  Since $x_1 - x_0 \ne 0$, the linear system
  \[
    \spalignmat{1,x_0; 1,x_1}\spalignmat{s;t}=\spalignmat{y_0;y_1},
  \]
  has a solution, say $(s_0, t_0)^T$. Thus,
  \begin{eqnarray*}
    \sigma(s_0+t_0x_0) &=& \sigma(y_0) = \epsilon/2 < \epsilon, \hspace{12pt}\mbox{ and }\\
    \sigma(s_0+t_0x_1) &=& \sigma(y_1) = 1-\epsilon/2 > 1-\epsilon.
  \end{eqnarray*}

  Suppose, in addition, that $x_0<x_1$ and $\epsilon<1/2$. 
  Since $\sigma(s+tx)$ is monotone and $\sigma(s+tx_0)<\epsilon<1/2<1-\epsilon<\sigma(s+tx_1)$,
  it follows that $\sigma(s+tx)$ is increasing, which proves the remainder of the lemma.
\end{proof}
  
The next lemma asserts that we can separate points from closed sets using
functions in $\NN_2$ in a strong sense. That is, one hidden layer suffices to 
separate points from closed sets in the described sense.

\begin{lem} \label{lem:ptset}
  Let $B\subset K$ be a closed set, and $\vx_0\in K-B$.
  For each $\epsilon>0$ there exists $g\in\NN_2$ such that
  $g> 1-\epsilon$ on $B$ and $g(\vx_0)<\epsilon$.
\end{lem}
\begin{proof}
  Without loss of generality, assume $0<\epsilon<1/3$.
  Let $\vb\in B$. Since $\vb\ne\vx_0$, there exists $f_\vb\in\NN_1$ such that
  $f_\vb(\vx_0)<\epsilon/2$ and $f_\vb(\vb)>1-\epsilon/2$.
  Let 
  \[
    U_\vb = \{\vx\in K \st f_\vb(\vx)>1-\epsilon\}.
  \]
  Since $f_\vb$ is continuous, $U_\vb$ is open and $\vb\in U_\vb$,
  so $\{U_\vb\}_{\vb\in B}$ is an open cover of the compact set $B$.
  Let $\vb_1,\ldots,\vb_N\in B$ such that $\{U_{\vb_1},\ldots, U_{\vb_N}\}$
  is a cover of $B$.

  By Lemma \ref{lem:ptpt}, there exist $s, t\in\R$ such that $\sigma(s+tx)<\epsilon/N$
  on $(-\infty, \epsilon)$ and $\sigma(s+tx)>1-\epsilon$ on $(1-\epsilon, \infty)$.
  For each $1\le j\le N$, let $F_j=\sigma(s + tf_{\vb_j})$. Then $F_j\in\NN_1^\sigma$
  and $F_j(\vx_0)<\epsilon/N$ and $F_j\ge 1-\epsilon$ on $U_{\vb_j}$.

  Let $g = \sum_{j=1}^N F_j$, and it follows that $g\in\NN_2$,
  $g(\vx_0)<\epsilon$ and $g\ge 1-\epsilon$ on $B$. 
\end{proof}

The following lemma is the major tool in the proof of the universal approximation
theorem we will present. It asserts that two hidden layers suffice to separate
disjoint closed sets, in a sense very similar to the previous lemma.
Its proof is quite similar to the previous one.

\begin{lem} \label{lem:setset}
  Let $A$ and $B$ be disjoint closed subsets of $K$. Then for each $\epsilon>0$,
  \begin{enumerate}
    \item[(i)] there exists $h\in\NN_3$ such that $h<\epsilon$ on $B$ and $h>1-\epsilon$ on $A$,
    \item[(ii)] there exists $H\in\NN_3^\sigma$ such that $0\le H < \epsilon$ on $B$ and $1-\epsilon < H \le 1$ on $A$.
  \end{enumerate}
\end{lem}
\begin{proof}
  Again without loss of generality we assume $\epsilon\in(0,1/3)$.
  For each $\va\in A$, by Lemma \ref{lem:ptset}, there exists
  $\widetilde{g}_\va\in\NN_2$ such that $\widetilde{g}_\va > 1-\epsilon/2$ on $B$ and
  $\widetilde{g}_\va(\va) < \epsilon/2$.
  Let $g_\va = 1-\widetilde{g}_\va$, so that 
  $g_\va\in\NN_2$, and 
  $g_\va < \epsilon/2$ on $B$,
  and $g_\va(\va)>1-\epsilon/2$.
  Let
  \[
    U_\va=\{ \vx\in K \st g_\va(\vx) > 1-\epsilon\}.
  \]
  Since $g_\va$ is continuous, each $U_\va$ is open. And since
  $\va\in U_\va$, it follows that $\{U_\va\}_{\va\in A}$ is an
  open cover of the compact set $A$. 
  Let $\va_1,\ldots, \va_N\in A$ for which $\{U_{\va_1},\ldots, U_{\va_N}\}$
  is a cover of $A$.

  Let $s,t\in\R$ for which $\sigma(s+tx)<\epsilon/N$ on $(-\infty, \epsilon)$
  and $\sigma(s+tx)>1-\epsilon$ on $(1-\epsilon,\infty)$. Set
  \[
    h = \sum_{j=1}^N \sigma(s + tg_{\va_j}).
  \]
  First note that $h\in\NN_3$. If $\va\in A$, then $\va\in U_{\va_k}$ for some $k$,
  which implies that $g_{\va_k}(\va) > 1-\epsilon$ and so $h(\va)>1-\epsilon$.
  And if $\vb\in B$ it follows that $g_{\va_j}(\vb) < \epsilon/2$ for all $j$,
  which gives $\sigma(s + g_{\va_j}(\vb)) < \epsilon/N$ for all $j$ and 
  hence $h(\vb) < \epsilon$. This proves the first part.

  Since $\sigma$ is increasing, by Lemma \ref{lem:ptpt}
  there exist $s,t\in\R$ such that $\sigma(s+tx)<\epsilon$
  on the interval $(-\infty, \epsilon)$ and $\sigma(s+tx)>1-\epsilon$
  on the interval $(1-\epsilon, \infty)$. Then $H=\sigma(s+th)$ has the
  required properties.

\end{proof}

  Recall that if $K$ is a compact subset of $\R^n$ and
  $f:K\longrightarrow\R$ is continuous, then the
  {\em sup norm} of $f$, denoted $\|f\|_u$,
  is 
  \[
    \|f\|_u = \sup\{ |f(\vx)| \st \vx\in K\}.
  \]

\begin{thm}[A Universal Approximation Theorem] \label{thm:UAT}
  Let $\sigma$ be a 0-1 squashing function, and $\NN_k,\NN_k^\sigma$
  as previously defined.
  Let $T:K\longrightarrow \R$ be a continuous function.
  For each $\epsilon>0$ there exists $f\in\NN_4$ such that
  $\|f-T\|_u<\epsilon$; that is, $\NN_4$ is dense in $C(K)$
  with respect to the $\sup$ norm.
\end{thm}
\begin{proof}
  By way of contradiction, suppose that $T:K\longrightarrow \R$
  is a continuous function and
  \[
    \inf_{f\in\NN_4} \|f - T\|_u = \alpha > 0.
  \]
  Let $\widehat{f}\in\NN_4$ with $\alpha\le \|\widehat{f}-T\|_u < 4\alpha/3$.
  Define
  \begin{eqnarray*}
    U^+ &=& \left\{ \vx\in K \st \frac{\alpha}{3}\le (\widehat{f}-T)(\vx) \le \frac{4\alpha}{3} \right\}, \\
    U^- &=& \left\{ \vx\in K \st -\frac{4\alpha}{3}\le (\widehat{f}-T)(\vx) \le - \frac{\alpha}{3} \right\}.
  \end{eqnarray*}
  Since $U^+, U^-$ are closed and disjoint,
  by Lemma \ref{lem:setset} there exists $H\in\NN_3^\sigma$ such that
  \[ 
    0\le H < \frac{1}{6} \hspace{12pt}\mbox{ on } U^-, \hspace{12pt}\mbox{ and }
    \frac{5}{6} < H \le 1\hspace{12pt}\mbox{ on } U^+.
  \]
  Consider $f=\widehat{f}-\alpha H + \frac{\alpha}{2}\in\NN_4$.
  We claim that $\|f-T\|_u < \alpha$.

  Suppose first that $\vx\in U^+$. Then
  \begin{eqnarray*}
    (f-T)(\vx) &=& (\widehat{f}-T)(\vx) - \alpha H(\vx) + \frac{\alpha}{2} \\
               &<& \frac{4\alpha}{3} - \frac{5\alpha}{6} + \frac{\alpha}{2} 
               = \alpha.
  \end{eqnarray*}
  We also have that
  \begin{eqnarray*}
    (f-T)(\vx) &=& (\widehat{f}-T)(\vx) - \alpha H(\vx) + \frac{\alpha}{2} \\
               &\ge & \frac{\alpha}{3} - \alpha + \frac{\alpha}{2} 
               = -\frac{\alpha}{6} > -\alpha,
  \end{eqnarray*}
  And therefore $|f-T|<\alpha$ on $U^+$. Suppose now that $\vx\in U^-$.
  Then
  \begin{eqnarray*}
    (f-T)(\vx) &=& (\widehat{f}-T)(\vx) - \alpha H(\vx) + \frac{\alpha}{2} \\
               &\le & -\frac{\alpha}{3} + \frac{\alpha}{2} = \frac{\alpha}{6} < \alpha,
  \end{eqnarray*}
  and
  \begin{eqnarray*}
    (f-T)(\vx) &=& (\widehat{f}-T)(\vx) - \alpha H(\vx) + \frac{\alpha}{2} \\
               &>& -\frac{4\alpha}{3}  - \frac{\alpha}{6} + \frac{\alpha}{2} 
                = - \alpha,
  \end{eqnarray*}
  and hence $|f-T|<\alpha$ on $U^-$.
  Finally, suppose that $\vx\in K-(U^-\cup U^+)$.
  Since $-\frac{4\alpha}{3} < (\widehat{f}-T)(\vy) < \frac{4\alpha}{3}$
  for all $\vy\in K$, it follows that $-\frac{\alpha}{3} < (\widehat{f}-T)(\vx) < \frac{\alpha}{3}$.
  From this and $0\le H\le 1$ it follows that
  \begin{eqnarray*}
    (f-T)(\vx) &=& (\widehat{f}-T)(\vx) - \alpha H(\vx) + \frac{\alpha}{2} \\
               &<& \frac{\alpha}{3} + \frac{\alpha}{2} = \frac{5\alpha}{6} < \alpha,
  \end{eqnarray*}
  and
  \begin{eqnarray*}
    (f-T)(\vx) &=& (\widehat{f}-T)(\vx) - \alpha H(\vx) + \frac{\alpha}{2} \\
               &>& -\frac{\alpha}{3}  - \alpha + \frac{\alpha}{2} 
                = - \frac{5\alpha}{6} > -\alpha.
  \end{eqnarray*}
  But $f\in\NN_4$ and $\|f-T\|_u<\alpha$ is a contradiction, and so the theorem is proven.
\end{proof}

\section{Remarks}
  There are a number of easily deduced corollaries
  that extend the result. 
  For example, if $\sigma$ is strictly increasing, we can obtain
  a similar result using $\NN_4^{\sigma}$ for continuous
  functions from $K$ to the image $\sigma(\R)$.
  Of course another corollary is that
  functions $C(K,\R^m)$ can be uniformly approximated 
  by functions in $(\NN_4)^m$, and these two extensions could
  also be combined. But these are all very standard and obvious extensions
  which apply quite generally.

  There are several places where this argument is `wasteful'; that is,
  places where we are not using the full power of a three hidden layer neural
  network. This leaves room for the possibility that a very similar
  technique might be used to deduce the same conclusion for neural networks
  with two hidden layers.

  We are grateful to Anupam Pal Choudhury for carefully reading the first
  draft of this note and pointing out several places where the argument
  was unclear.

\bibliographystyle{plain}

\end{document}